\documentclass{article}[9pt] 

\usepackage{spconf}
\usepackage{cite}

\usepackage[dvips]{graphicx}
\usepackage{epsfig,epsf,epstopdf,graphicx}
\usepackage[cmex10]{amsmath}
\usepackage{amssymb}
\usepackage{bbm}
\usepackage{amsthm }
\usepackage{nicefrac}
\usepackage{hyperref}
\usepackage{xcolor}
\hypersetup{colorlinks=true}
\usepackage{tabularx}

\newtheorem{lemma}{Lemma}

\theoremstyle{theorem}

\theoremstyle{definition}

\theoremstyle{plain}

\theoremstyle{plain}

\newcommand{\nuv}{{\mbox{\boldmath $\nu$}}}
\newcommand{\xiv}{{\mbox{\boldmath $\xi$}}}

\newcommand{\bv}{{\bf{b}}}

\newcommand{\qv}{{\bf{q}}}
\newcommand{\rv}{{\bf{r}}}

\newcommand{\wv}{{\bf{w}}}
\newcommand{\xv}{{\bf{x}}}

\newcommand{\zerov}{{\bf{0}}}
\newcommand{\onev}{{\bf{1}}}

\newcommand{\Db}{\mathbf{D}}
\newcommand{\Eb}{\mathbf{E}}
\newcommand{\Fb}{\mathbf{F}}
\newcommand{\Gb}{\mathbf{G}}
\newcommand{\Hb}{\mathbf{H}}
\newcommand{\Ib}{\mathbf{I}}
\newcommand{\Jb}{\mathbf{J}}
\newcommand{\Kb}{\mathbf{K}}
\newcommand{\Lb}{\mathbf{L}}
\newcommand{\Mb}{\mathbf{M}}

\newcommand{\Sb}{\mathbf{S}}

\newcommand{\Xb}{\mathbf{X}}
\newcommand{\Yb}{\mathbf{Y}}

\newcommand{\Real}{\mathbb{R}}

\newcommand{\Ec}{{\cal{E}}}
\newcommand{\Dc}{\Delta}
\newcommand{\Gc}{{\cal{G}}}

\newcommand{\Nc}{{\cal{N}}}

\newcommand{\Oc}{{\cal{O}}}
\newcommand{\Lc}{{\cal L}}
\newcommand{\Pc}{{\cal P}}
\newcommand{\Vc}{{\cal V}}

\newcommand{\norm}{\Vert}
\newcommand\normop[1]{\left\lVert#1\right\rVert}

\newcommand{\maxx}{\mathop{ \mathrm{max}}}

\newcommand{\argminn}{\mathop{ \mathrm{argmin}}}

\newcommand{\tr}{\mathop{ \mathrm{Tr}}}
\newcommand\asto[1]{#1^\ast}
\newcommand{\vecc}{\mathrm{vec}}

\newcommand{\Diag}{\mathrm{Diag}}
\newcommand{\cte}{\mathrm{const}}

\usepackage{algorithm}
\usepackage{multirow}
\usepackage{algcompatible}
\usepackage[bottom]{footmisc}
\usepackage{multicol}
\usepackage{balance}

\usepackage{tikz}
\usetikzlibrary{calc}

\usepackage{subcaption}
\usepackage{fixltx2e}
\usepackage{afterpage}
\usepackage{float}
\usepackage{stfloats}

\usepackage[space]{grffile}

\title{Joint Signal Recovery and  Graph Learning from Incomplete Time-Series
}

\name{Amirhossein~Javaheri $^{\star \dagger}$ \qquad Arash~Amini $^{\star}$ \qquad Farokh~Marvasti $^{\star}$ \qquad Daniel~P.~Palomar $^{\dagger}$\vspace*{-6pt}}
  
\address{$^{\star}$ Sharif University of Technology \\
  $^{\dagger}$ Hong Kong University of Science and Technology}

\begin{document}

\maketitle
\begin{abstract}
  Learning a graph from data is the key to  taking advantage of  graph signal processing tools. Most of the conventional algorithms for graph learning  require  complete data statistics, which might not be available in some scenarios. In this work, we aim to learn a graph from incomplete time-series observations. From another viewpoint, we consider the problem of semi-blind  recovery of time-varying graph signals where the underlying graph model is unknown.
We propose an algorithm based on the method of block successive upperbound minimization (BSUM), for simultaneous inference of the signal and the graph from incomplete data.
 Simulation results on synthetic and real time-series demonstrate the performance of the proposed method for  graph learning and signal recovery. 
\end{abstract}


%

\section{Introduction}
\label{sec:Intro}

Graph signal processing is an emerging field of research
with numerous applications in social networks \cite{campbell_social_2013}, 
neural networks \cite{wu_comprehensive_2021},  communications \cite{tamura_applications_1970}, 
and even financial markets \cite{cardoso_learning_2020}. 

Finding a graphical model for efficient signal representation  is a prerequisite  for 
graph signal processing.
An undirected graph models bilateral  correlations (e.g., factor graphs \cite{loeliger_introduction_2004}), whereas for unilateral dependencies, a directed graph is utilized (e.g., Bayesian networks  \cite{lalitha_decentralized_2019}). 
There are  numerous approaches to learning  the topology of the graph that best represents  the data, including stochastic approaches via a Gaussian Markov random field (GMRF) model \cite{friedman_sparse_2008, egilmez_graph_2017, zhao_optimization_2019} or  deterministic approaches incorporating measures of smoothness or stationarity for signal representation  \cite{kalofolias_how_2016}.
Another category of methods {\color{black}aims} to recover a signal  from corrupted (noisy or incomplete) measurements using a known graph model. 
Many of these algorithms exploit a graph-induced fidelity criterion, e.g., spatio-temporal smoothness \cite{mao_spatio-temporal_2018}   or  total variation \cite{chen_signal_2015, berger_graph_2017}
to help recover the graph  signal.

Graph learning algorithms rely on the integrity (completeness) of the data, whereas for   graph signal recovery methods,  the graph model of the data needs to be given a priori. 
In practical applications, however, the data may be incomplete,  or the graphical model may be unavailable.  
In this work, we study the problem of graph learning from incomplete  data or semi-blind graph signal recovery, in which the underlying graph structure is unknown.
{\color{black}Our method,  applied to both i.i.d. and time-dependent data, is shown to improve the reconstruction quality of missing entries in the signal by joint estimation of the signal and graph}.

The organization of the paper is as follows. In Section \ref{sec:problem}, we  formulate the problem and in Section \ref{sec:proposed}, we propose an iterative solution method. The simulation results are provided in Section \ref{sec: Simulation}.
%
%
For the notations, we reserve bold lower-case letters for  vectors (e.g., $\xv$) and bold   upper-case letters for  matrices (e.g., $\Xb$). The Kronecker  and the Hadamard 
products are 
respectively  denoted with $\otimes$ and $\odot$. In addition, $^{\circ }$ and $\oslash$, are respectively used for  element-wise power and division.

\subsection{Problem Statement}
\label{sec:problem}
An undirected graph with $n$ vertices can be represented by $\Gc = \{\Vc, \Ec, \wv \} $, with $ \Vc =\{1,…,n\}$ being the set of vertices, $\Ec \subseteq \{ \{i, j\} \vert\, i, j \in \Vc \} $ the set of edges, and  $\wv\in\Real ^ {n (n-1) / 2} $ the edge weights.  
{\color{black}
Let $\Xb^\ast = [\xv^\ast_1,\hdots,\xv^\ast_T]\in \Real^{n\times T}$ denote the matrix of the original signal for $T$ time-stamps.
Assume we have   observations with missing entries as $
 \Yb=\Mb \odot \Xb^\ast
 $, 
 where
   $\Mb$ denotes the binary sampling mask matrix.
Now, given $\Yb$ and $\Mb$, the problem considered in this paper is   to estimate  the original  signal $\Xb^\ast$ and the weights $\wv^\ast$ of an undirected graph model that  encode how similarly the signal elements vary with time (assuming spatial smoothness for the temporal variations of the signal).}

\section{Proposed Method}
\label{sec:proposed}
Assume the graph is connected. 
Using the graph learning framework in \cite{egilmez_graph_2017, zhao_optimization_2019} 
{\color{black} and assuming a Laplacian GMRF model for the temporal difference of the signal inspired by the notion of spatio-temporal smoothness \cite{mao_spatio-temporal_2018},}
we propose to solve the following problem for joint estimation of $\Xb^\ast$ and $\wv^\ast$
\begin{align}
\label{eq:MAP_X_A_w}
\Xb^\ast,  \wv^\ast =& \argminn_{\Xb, \wv\geq \zerov} f(\Xb,  \wv) \\ 
f(\Xb,  \wv)\triangleq & \normop{\Yb-\Mb \odot \Xb}_F^2 + \alpha \tr \left(\Lc(\wv)\Dc(\Xb) \Dc(\Xb)^\top\right)  \nonumber \\ 
&-\beta  \log \det (\Lc(\wv) + \Jb)+  \gamma \normop{\wv}_1 \nonumber 
\end{align}
where $\Dc(\Xb) = \Xb - \Xb\Db = [\xv_1-\xv_0,\hdots,\xv_T-\xv_{T-1}]$ represents the first-order temporal differences of the signal, and $\Db\in \Real^{T\times T}$ is an upper triangular matrix with only $1$s as  the {\color{black}first} upper diagonal elements (assuming $\xv_0 = \zerov$). Also,   $\Jb = (1/n) \onev \onev^\top$, and $\Lc\!: \Real^{n(n-1)/2} \rightarrow \Real^{n\times n}$ denotes the Laplacian operator \cite{kumar_unified_2020} that maps $\wv$ 
to 
the Laplacian matrix. The adjoint of this operator is also  denoted with $\asto \Lc$. The first term in \eqref{eq:MAP_X_A_w} is in fact the similarity criterion, and the second term quantifies the spatio-temporal smoothness, which measures the smoothness of temporal variations in a graph signal. 
For temporally i.i.d. data, 
this simply leads to the conventional graph signal smoothness \cite{kalofolias_how_2016}, 
{\color{black}making our model applicable for both i.i.d. and time-dependent data.}
The third term, partially referred to as data-fidelity in \cite{egilmez_graph_2017},  appears in the maximum likelihood estimation of the Laplacian in a GMRF model.
The last term also acts as a {\color{black} sparsity-promoting} regularization function  on $\wv$.

To solve problem \eqref{eq:MAP_X_A_w}, 
we use the block successive upperbound minimization (BSUM) \cite{razaviyayn_unified_2013,sun_majorization-minimization_2017} method 
which is actually an extension of the block-coordinate-descent (BCD) 
\cite{tseng_convergence_2001}.
{\color{black} Here, we minimize an uppperbound or a majorizer  (see the definition in \cite{razaviyayn_unified_2013})  of the original cost with respect to a block variable in each iteration, 
  resulting in unique and closed-form solutions to the subproblems (with guaranteed convergence \cite{razaviyayn_unified_2013}).
These update steps are as follows:}

\subsection{$\Xb$-update step}

Let $\wv$ be fixed. Then $f(\Xb, \wv)$   in \eqref{eq:MAP_X_A_w}, would be only a function  of $\Xb$  as  
\begin{align}
\label{eq:X-sub-orig}
f_{\Xb}(\Xb) =&    \tr \left( (\Yb-\Mb \odot \Xb)(\Yb-\Mb \odot \Xb)^\top\right)+ \nonumber \\
 &\alpha \tr \left(\Lc(\wv) \Dc(\Xb)\Dc(\Xb)^\top\right) + \cte.
\end{align}
Using vectorial representation,
one obtains
\begin{align}
f_{\Xb}(\Xb) =&\,    \vecc(\Xb)^\top \Gb \vecc(\Xb) -2 \vecc(\Xb)^\top \bv + \cte \nonumber
\end{align}
where
$
\Gb =   \Diag(\vecc(\Mb)) + \alpha \Hb^\top (\Ib_{T} \otimes \Lc(\wv))\Hb
$, $\Hb  =\Ib_{nT} - \Db^\top\otimes \Ib_n$,
 and $\bv = \vecc(\Yb)$.

\begin{lemma}\label{lemma:X_majorization}
Assume $\Xb_0$ to be constant. The function
\begin{align}
f_{\Xb}^S(\Xb; \!\Xb_0) = f_{\Xb}(\Xb) + \vecc(\Xb - \!\Xb_0)^{\!\top} \!(\theta \Ib - \Gb) \vecc(\Xb - \!\Xb_0)  \nonumber
\end{align}
is a strictly convex majorizer  for $f_{\Xb}(\Xb)$ if  $\theta > 1  +  4\alpha\normop{ \Lc(\wv)}$. 
\end{lemma}

\begin{proof}
For $\theta >  \lambda_{\max}(\Gb) = \normop{\Gb}$, the matrix $ \theta \Ib - \Gb$ is positive definite. This implies that $f^S_{\Xb}(\Xb; \Xb_0)\geq f_{\Xb}(\Xb)$, and the equality is  only occurred at $\Xb = \Xb_0$. 
In addition, we can obtain an upper-bound for  $\normop{\Gb}$ as
\begin{align}
\normop{\Gb} &=\maxx_{\normop{\xv}=1} \normop{ \Diag(\vecc(\Mb)) \xv + \alpha \Hb^\top(\Ib  \otimes \Lc(\wv))\Hb  \xv} \nonumber \\
& \leq \maxx_{\normop{\xv}=1}  \normop{\Diag(\vecc(\Mb)) \xv} + \alpha \normop{\Hb^\top(\Ib  \otimes \Lc(\wv)) \Hb\xv}\nonumber \\
&\leq\maxx_{\normop{\xv}=1} \normop{ \xv} + \alpha\normop{\Hb}^2\normop{\Ib  \otimes \Lc(\wv)}\normop{ \xv} \nonumber\\
&= 1  + \alpha \normop{\Hb}^2\normop{ \Lc(\wv)} \nonumber
\end{align}
It is also easy to show that $\normop{\Hb} \leq 2 $.
Therefore, 
 $f_{\Xb}^S(\Xb; \Xb_0)$ would be a majorization function for $f_{\Xb}(\Xb)$ if  $\theta >  1  + 4\alpha \normop{ \Lc(\wv)} \geq   \normop{\Gb}$. We also obtain
\begin{align*}
f_{\Xb}^S(\Xb; \Xb_0) 
=&\,  \theta \normop{\vecc(\Xb-\Xb_0) +\frac{\Gb   \vecc(\Xb_0) -\bv}{\theta}}^2 + \cte
\end{align*}
Hence, $f_{\Xb}^S(\Xb; \Xb_0)$ is trivially  a strictly convex (quadratic) function  whose unique minimizer is given by
$\vecc(\Xb_0) -\frac{1}{\theta}\left( \Gb   \vecc(\Xb_0) -\bv \right)  = \vecc\left(\Xb_0 - \frac{1}{2\theta}\frac{\partial}{\partial \Xb} f_{\Xb}(\Xb)\vert_{\Xb_0}\right)$.
 \end{proof}

 Setting $\Xb_0 = \Xb^{(j)}$,  the $\Xb$-update step yields as follows
\begin{align}
\label{eq:X_update}
\Xb^{(j+1)} &= \argminn_{\Xb} f_{\Xb}^S(\Xb; \Xb^{(j)}) =  \Xb^{(j)} -\tfrac{1}{2\theta} \tfrac{\partial}{\partial \Xb} f_{\Xb}(\Xb^{(j)})\nonumber\\
\tfrac{\partial}{\partial \Xb}f_{\Xb}(\Xb)&=2\left( \alpha \Lc(\wv) \Dc(\Xb)(\Ib-\Db^\top)+  \Mb \odot \Xb - \Yb \right)
\end{align}

\subsection{$\wv$-update step}
Since $\wv\geq \zerov$, we may write $2\normop{\wv}_1 = \tr(\Lc(\wv)\Hb_{\text{off}})$, where $\Hb_{\text{off}} = \Ib - \onev\onev^\top$
Hence, assuming  $\Xb$ to be fixed, after division by $\beta$, the cost function in \eqref{eq:MAP_X_A_w} reduces to 
  \begin{align}
f_{\wv}(\wv) = \tr(\Lc ( \wv) \Kb) -  \log \det( \Lc(\wv)+\Jb )
 \end{align}
  where $\Kb =\frac{1}{\beta} \left(\alpha \Dc(\Xb)\Dc(\Xb)^\top+\gamma /2  \Hb_\text{off}\right) $.

\begin{lemma}
\label{prop:w_surr}
Let  $\qv = \Lc^\ast\left( (\Lc(\wv^{(j)}) + \Jb)^{-1}\right)$,  $\rv =  \Lc^\ast(\Kb)$,  and $\tau>0$  be a constant. Also assume $\wv_0 \geq 0$. Then, the following is  a strictly convex majorization function  for $f_{\wv}(\wv)$. 
\begin{align}
\label{eq:k}
f^S_{\wv}(\wv;\wv_0) = \tau\langle  \qv \odot \wv_0^{\circ 2}, \wv \oslash \wv_0 + (\wv_0 + 1/\tau)\oslash \wv - 2\rangle+\nonumber\\ \langle\wv,\rv \rangle+ \tr\left((\Lc(\wv_0) + \Jb)^{-1} \Jb\right)  - \log \det (\Lc(\wv_0)+\Jb) - n
\end{align}
\end{lemma}
\label{App:prop1}
\begin{proof}
Taking advantage of the notion of the Laplacian operator \cite{kumar_unified_2020}, we have
\begin{align}
\Lc(\wv) + \Jb  &= \Eb \Diag (\wv)  \Eb^\top+ \Jb =  \Gb \Diag(\tilde{\wv}) \Gb^\top 
\end{align}
where   $\tilde{\wv} = [\wv^\top 1/n]^\top$ and $\Gb = [\Eb, \onev]$. The matrix 
  $\Eb = [\xiv_{1} , \hdots , \xiv_{n(n-1)/2}]\in \Real^{n\times  n(n-1)/2}$,  {\color{black} consists} of vectors $\xiv_{k}$ for $ k=i-j + \frac{j-1}{2}(2n-j), \,\, i>j$, each of which has  a $+1$ at the $j$-th position, a $-1$ at the $i$-th position, and zeros elsewhere. 
Since the $\log \det$ function is concave, an upperbound for  $-\log \det (\Lc(\wv)+\Jb)$ can be constructed via the following inequality \cite{zhao_optimization_2019} 
\begin{align}
\label{eq:fw_major_term1}
-\log \det( \Lc(\wv)+\Jb) \leq &\tr\left(\Fb_0(\Gb \Diag(\tilde{\wv}) \Gb^\top )^{-1} \right)  \\
& -\log \det (\Lc(\wv_0)+\Jb) - n \nonumber 
\end{align}
where $\Fb_0 =\Lc(\wv_0) + \Jb$, and the equality is only achieved  at $\wv=\wv_0$.
Moreover,  using Lemma 4 in  \cite{zhao_optimization_2019}, one can obtain another majorizer as follows:
\begin{align}
\label{eq:fw_major_term2}
&\tr\left(\Fb_0(\Gb \Diag(\tilde{\wv}) \Gb^\top )^{-1} \right)  \\
&\leq  \tr\left( \Fb_0^{-1} \Gb \Diag(\tilde{\wv}_0^{\circ 2} \oslash \tilde{\wv} ) \Gb^\top \right)  \nonumber\\
&=\langle \wv_0^{\circ 2} \oslash \wv, \asto \Lc(\Fb_0^{-1} )\rangle + \tr\left(\Fb_0^{-1} \Jb\right)  \nonumber
\end{align}

Now, let $g_{\wv}(\wv;\wv_0) = \tr(\Lc(\wv) \Kb) + \langle \wv_0^{\circ 2} \oslash \wv, \asto \Lc(\Fb_0^{-1} )\rangle$. Also, define
$\rv = \asto \Lc(\Kb)$ and
$\qv =  \asto \Lc(\Fb_0^{-1}) =  \asto \Lc(\Lc(\wv_0 + \Jb)^{-1})$.
Then,  $g_{\wv}(\wv;\wv_0)$ can be decomposed to scalar functions of $w_i$ as 
\begin{align}
g_{\wv}(\wv;\wv_0) \,= \!\sum_{i=1}^{n(n-1)/2} g_{w_i}(w_i;{w_0}_i) =r_i w_i  + q_i\frac{{w_0}_i^2}{w_i}
\end{align}
with $r_i = [\rv]_i,\quad q_i = [\qv]_i$. Now, let   $h(x) = x + \frac{1}{x} -2 $. It can be easily shown that $h(x)$ is always non-negative for $x>0$, and  only achieves zero  at $x=1$.
In addition, since $\Lc(\wv^{(j)}) + \Jb\succ 0$ and $\Kb\succeq 0$, one can conclude that  $q_i>0$  and $r_i\geq 0$ (using the property of the adjoint Laplacian operator).
Thus, one may suggest a majorizer for $g_{w_i}(w_i; {w_0}_i)$  as
\begin{align*}
g^S_{w_i}(w_i; {w_0}_i) &= g_{w_i}(w_i; {w_0}_i) +  \tau q_i {w_0}_i^2\,\, h(w_i/{w_0}_i) \\
& =r_i w_i+  \tau q_i {w_0}_i^2  \left(\frac{w_i}{{w_0}_i}  +  \frac{{w_0}_i+1/\tau}{w_i}-2\right)
\end{align*}
with $\tau$ being a positive constant.
Finally, after simplifications, one obtains the  majorizer in \eqref{eq:k} for $f_\wv(\wv)$, using  \eqref{eq:fw_major_term1} and \eqref{eq:fw_major_term2}.  
It is easy to verify that $ f_{\wv}^S(\wv; \wv^{(j)})$ is strictly convex for $\wv \geq \zerov$ (via second order derivatives). 
\end{proof}
Using the above lemma, the $\wv$-update step yields  as
\begin{align}
\label{eq:w_bsum}
\wv^{(j+1)} &= \argminn_\wv \, f_{\wv}^S(\wv; \wv^{(j)}) \\
&= \wv^{(j)}  \odot \sqrt{(\tau\wv^{(j)}\odot\qv+\qv)  \oslash   (\tau\wv^{(j)}\odot \qv + \rv) }. \nonumber 
 \nonumber 
\end{align}

\subsection{Main algorithm}
The proposed method with all steps  can {\color{black} be} summarized in Algorithm \ref{algorithm:BSUM_joint_L_X_A}.   We choose $\Xb^{(0)} = \Yb$ and $\wv^{(0)} = \Pc_{\wv\geq \zerov}(\Sb_Y^\dagger)$ for  initialization, where $\Sb_Y = \frac{1}{T} \Yb \Yb^\top$ and  $\Pc_{\wv\geq \zerov}$ denotes the projection onto the set $\wv\geq \zerov$. The stopping criterion is met when the relative error between consecutive iterations becomes smaller than a threshold or when the  number of iterations exceeds a limit.

\begin{algorithm}[t]
\caption{Proposed algorithm to solve problem \eqref{eq:MAP_X_A_w}}
\vspace*{1ex}
\begin{algorithmic}
\State \textbf{Input:}   $\Yb$, $\Mb$. \quad \textbf{Parameters:}   $\alpha$, $\beta$, $\gamma$, and $\tau$. 
\State \textbf{Output:} $\Xb^{(j)}$, $\Lb^{(j)} = \Lc(\wv^{(j)})$.
\State \textbf{Initialization:} $\Xb^{(0)} = \Yb$, $\wv^{(0)} = \Pc_{\wv\geq \zerov}(\Sb_Y^\dagger)$, $j=0$
\REPEAT
\STATE   Obtain $\Xb^{(j+1)}$ using  \eqref{eq:X_update} with $\wv= \wv^{(j)}$.
\STATE  Obtain $\wv^{(j+1)}$ via 
\eqref{eq:w_bsum} 
with $\Xb= \Xb^{(j+1)}$.
\STATE Set $j \leftarrow j+1$
\UNTIL {a stopping condition is met}
\end{algorithmic}
\label{algorithm:BSUM_joint_L_X_A}
\end{algorithm}

\subsection{Computational complexity}
\label{sec:comp_cplx}
 The update step in \eqref{eq:X_update} is $\Oc(n^2T + T^2n  + n^3)$  computationally complex. Furthermore,
given $\Kb$, the  complexity of the $\wv$ update step \eqref{eq:w_bsum}  is controlled by $(\Lc(\wv^{(k)}) + \Jb)^{-1}$, which needs $\Oc(n^3)$ operations. It also costs $\Oc(n^2T+T^2n)$ operations to compute $\Kb$. 
Hence, each iteration of Algorithm \ref{algorithm:BSUM_joint_L_X_A} is $\Oc(n^3+n^2T+T^2n)$ computationally complex.


\section{Simulation Results}
\label{sec: Simulation}
Here, we present the simulation results of our proposed algorithm for graph learning and missing sample recovery on synthetic and real data. 

\subsection{Synthetic data}
For generation of  synthetic data, we consider   $n = 64$ and  $T = 640$. We use the Stochastic Block Model for the underlying graph
comprising of 4  clusters (blocks),  with   inter-cluster  and intra-cluster edge probabilities of 0.7 and 0.075, respectively.  
The  Laplacian matrix of the graph is then scaled
to  have  $\tr(\Lb^\ast) = n$. We then   generate random samples of the  signal via 
$
\xv^\ast_t  =   \sqrt{{\Lb^\ast}^\dagger} \nuv_t, \quad   \nuv_t\sim \Nc(\zerov, \Ib)
$.
The    original  data matrix $\Xb^\ast$ is then constructed as  $\Xb^\ast = [\xv^\ast_1, \hdots,\xv^\ast_T]$. 
Next, we normalize $\Xb^\ast$, so that each row has zero mean   and unit standard deviation.
 Finally, 
 the observations are obtained as $\Yb = \Mb \odot \Xb^\ast $, where $\Mb$ is the binary (sampling) mask matrix. 
Now, we provide the matrices $\Yb$ and $\Mb$ as inputs, for signal and  graph inference. The hyper-parameters of our method are chosen as $\alpha = 0.02$, $\beta = 0.02T$, $\gamma = 0.002T$, and $\tau=100$.
To measure the performance of the   graph learning algorithms, we use the relative error (RelErr) and the F-score criteria.
Let $\Lb^\ast\in \Real^{n\times n}$ be the ground-truth Laplacian, and $\widehat{\Lb}\in \Real^{n\times n}$   be the estimated one,
the relative error and the F-score values are defined as
\begin{align}
\text{RelErr} =  \dfrac{\norm \asto{\Lb} - \widehat{\Lb} \norm_F}{\norm \asto{\Lb}\norm_F},
\qquad
\text{F-score} = \dfrac{2 \text{TP} }{2\text{TP} + \text{FP} + \text{FN}}.  \nonumber
\end{align} 
The terms TP, FP, and FN, respectively denote the true positive, false positive, and false negative connections in the inferred graph.
The signal recovery   performance is also measured via SNR and the normalized mean squared error (NMSE) criteria, defined as
\begin{align}
\text{SNR}  = 20 \log_{10}\! \left(\tfrac{\norm \asto{\Xb}\norm_F} {\norm \asto{\Xb} - \widehat{\Xb} \norm_F} \right)\!, 
\quad
\text{NMSE}  =  \frac{1}{T}\sum_{i=1}^T \dfrac{\norm \asto{\xv}_i  - \hat{\xv}_i\norm^2}{\norm  \asto{\xv}_i\norm^2}
\nonumber
\end{align} 
where   $\asto{\Xb}$ and $\widehat{\Xb}$, denote the ground-truth and the estimated data matrices,  with $\asto{\xv}_i$ and  $\hat{\xv}_i$ being their $i$-th columns, respectively.

\subsubsection{Graph learning}\hfill

Here, we evaluate the performance of the proposed method for learning the Laplacian matrix from incomplete data. We compare 
the  results of our method 
with several benchmark algorithms for undirected graph learning. These  include the CGL \cite{egilmez_graph_2017}, 
the GSP toolbox graph learning methods \cite{kalofolias_how_2016}, namely the GSPBOX-Log and the GSPBOX-L2,
the GL-SigRep method in \cite{dong_learning_2016} for both graph learning and signal recovery assuming smooth signal representation,
and the 
nonconvex graph learning algorithm  in \cite{ying_nonconvex_2020} called NGL\footnote{\href{https://github.com/mirca/sparseGraph}{https://github.com/mirca/sparseGraph}}.
For fair comparison, the estimated Laplacian matrix is scaled  such that $\tr(\widehat{\Lb}) = n$. 

Figure \ref{fig_L_sr}  shows the estimation results  in terms of RelErr and F-score versus different values of the sampling rate (SR).   
As it is implied from the figure, the proposed algorithm  has superior performance in estimating the graph Laplacian matrix, especially at higher sampling rates.

\begin{figure}[t]
\centering
\hspace*{-4cm}
\begin{subfigure}{0.48\columnwidth}
   \includegraphics[trim=150 260 150 260,clip, width=\textwidth]{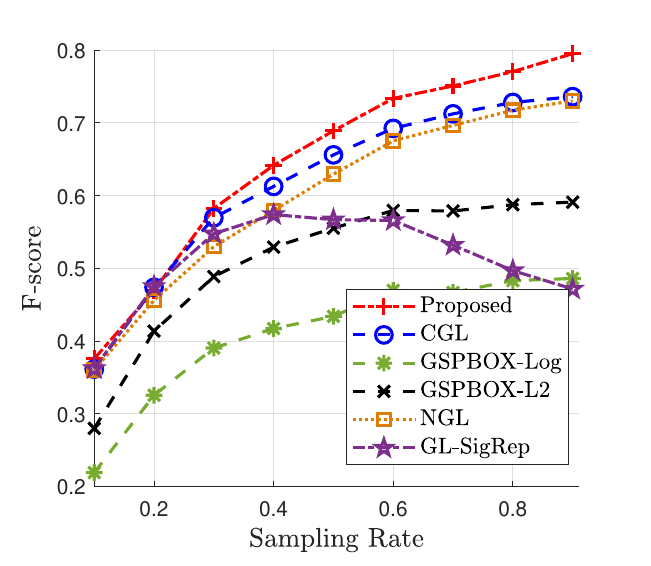}
\end{subfigure}
\begin{subfigure}{0.48\columnwidth}
    \includegraphics[trim=150 260 150 260,clip, width=\textwidth]{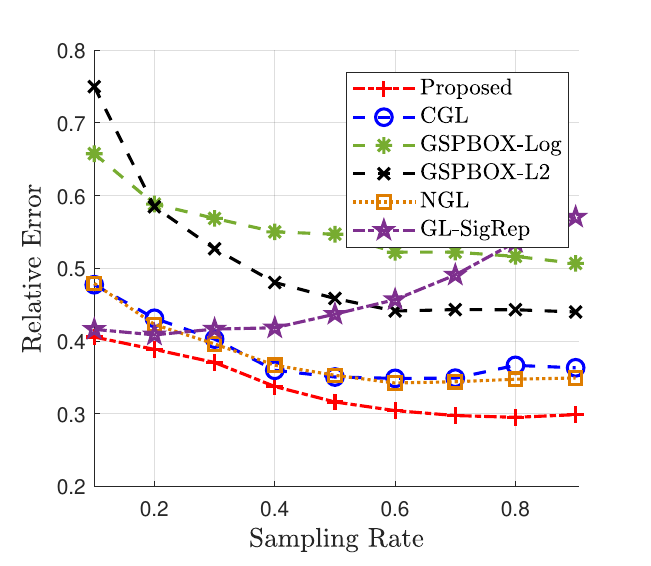}
\end{subfigure}
\hspace*{-4cm}
\caption{Performance results of the Laplacian matrix estimation from synthetic data (in terms of F-score and relative error) versus the sampling rate ($\mathrm{SR}$).}
\label{fig_L_sr}
\end{figure}

\subsubsection{Data matrix recovery}\label{sim_synth_X}\hfill

In this part, we evaluate the performance of the proposed method for recovery of the original signal  $\Xb^\ast$.
We  compare our method  with several benchmark algorithms, 
some of which are also based on graphical modeling. These  include the
SOFT-IMPUTE algorithm\footnote{\href{https://cran.r-project.org/web/packages/softImpute/index.html}{https://cran.r-project.org/web/packages/softImpute/index.html}} for matrix completion using nuclear norm regularization \cite{mazumder_spectral_2010},
the GL-SigRep method \cite{dong_learning_2016},
the method for joint inference of signals and graphs (JISG) \cite{ioannidis_semi-blind_2019}, 
the time-varying graph signal reconstruction method (TVGS) \cite{qiu_time-varying_2017} and the method in \cite{perraudin_towards_2016} named as Graph-Tikhonov.
For the last two methods, the underlying graph Laplacian matrix must be given a priori. For this, we use the CGL algorithm to infer the Laplacian matrix from the incomplete observations $\Yb$.
Figure \ref{fig_X_syn_sr} depicts the performance of 
 the proposed algorithm compared to the benchmark methods for recovery of the data matrix, 
at different values of the sampling
 rate. The proposed method is shown to outperform the other methods for recovery of the missing data.

\begin{figure}[t]
\centering \hspace*{-4cm}
\begin{subfigure}{0.48\columnwidth}
   \includegraphics[trim=150 260 150 260,clip, width=\textwidth]{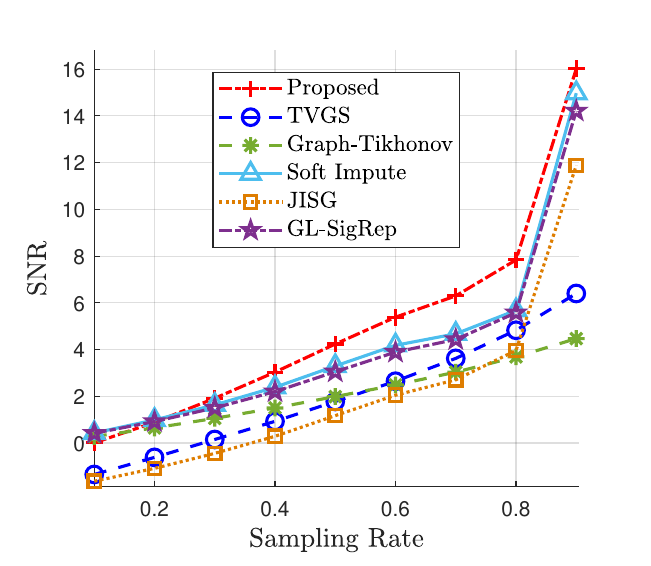}
\end{subfigure} 
\begin{subfigure}{0.48\columnwidth}
    \includegraphics[trim=150 260 150 260,clip, width=\textwidth]{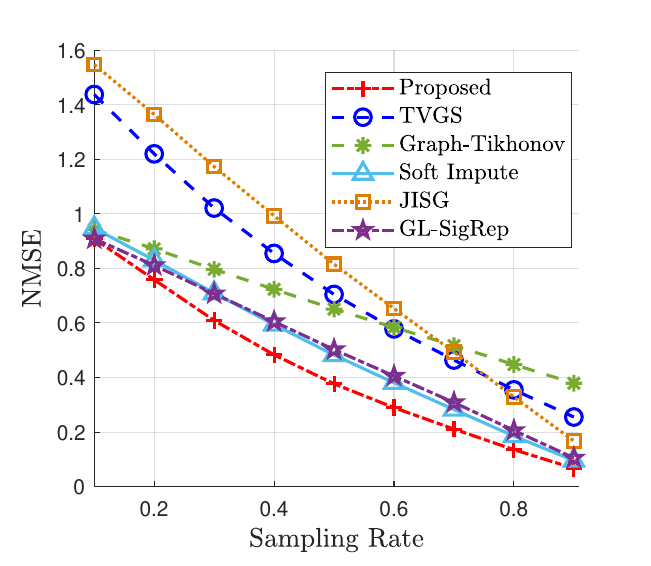}
\end{subfigure} 
\hspace*{-4cm}
\caption{Performance results of the original synthetic data matrix reconstruction (in terms of SNR and NMSE) versus the sampling rate.}
\label{fig_X_syn_sr}
\end{figure}

\subsection{Real data}
This part provides  numerical results on real  data. 
For this purpose, we use 
the data corresponding to the PM2.5 concentration in the air of the state of California. 
The  data matrix  of  dimension 93x300,  contains concentration values of 
PM2.5 particulate matters,  measured at 93 stations (in California) for 300 days starting from January 1, 2015. %
We then normalize the data as explained in the previous part
and construct
$\Yb = \Mb \odot\Xb^\ast $ (with random $\Mb$). 
Here, we only provide the results of graph signal recovery (since there is no ground-truth graph model).
Figure
\ref{fig_X_real3_sr} 
is an example to show that our proposed method has better performance  compared to the benchmark algorithms, in
recovery of real-world signals.

\section{Conclusion}
In this paper, we examined the problem of learning a graph from  incomplete data, which can also be considered as semi-blind recovery of missing samples of a time-varying graph signal. 
We proposed an algorithm to jointly  estimate the underlying graph model and the signal based on
the block successive upperbound minimization method. 
We further 
analyzed the computational complexity of our proposed method in Section \ref{sec:comp_cplx}.
The results of simulations on  synthetic and real data provided in Section \ref{sec: Simulation}, also demonstrate the efficiency of
the proposed method for both signal recovery  and graph inference from incomplete time-series observations.


\begin{figure}[!]
\centering \hspace*{-4cm}
\begin{subfigure}{0.48\columnwidth}
   \includegraphics[trim=150 260 150 260,clip, width=\textwidth]{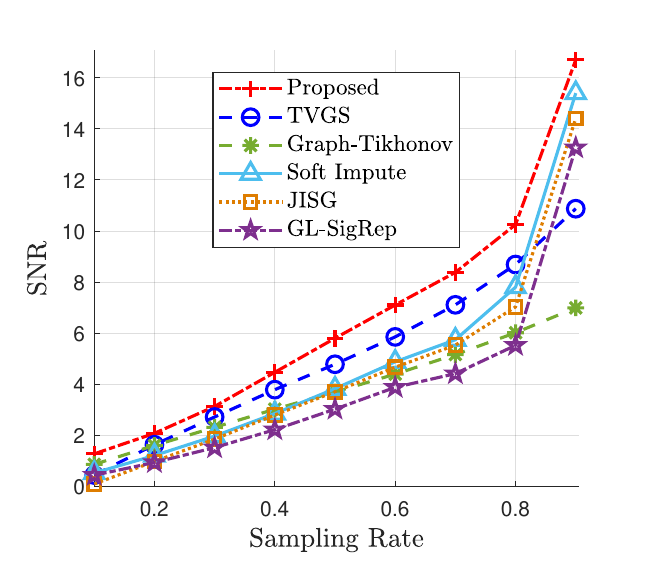}
\end{subfigure} 
\begin{subfigure}{0.48\columnwidth}
    \includegraphics[trim=150 260 150 260,clip, width=\textwidth]{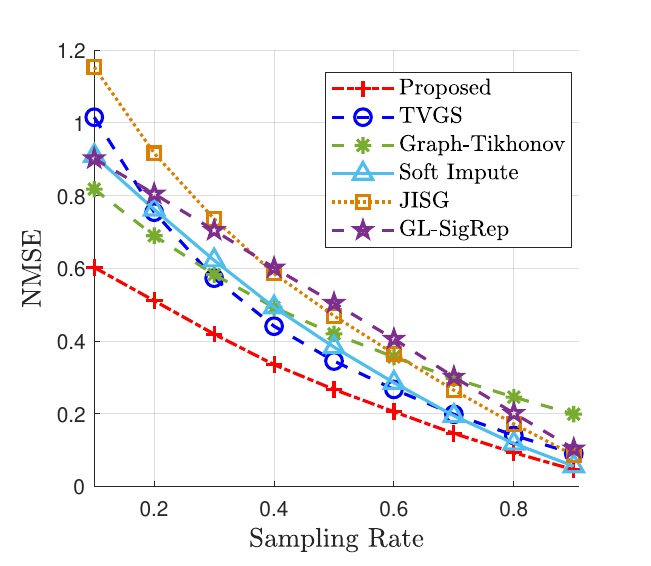}
\end{subfigure} 
\hspace*{-4cm}
\caption{Performance results of the recovery algorithms (in terms of SNR and NMSE) for reconstruction of PM2.5 data matrix from incomplete measurements,  at different  sampling rates.}
\label{fig_X_real3_sr}
\end{figure}

\bibliography{My_Library
}
\bibliographystyle{IEEEbib}


\end{document}